\documentclass[11pt]{article}

\usepackage[utf8]{inputenc}
\usepackage[T1]{fontenc}
\usepackage{amsmath,amssymb,amsthm}
\usepackage{mathtools}
\usepackage{geometry}
\usepackage[colorlinks=true,citecolor=blue,linkcolor=blue,urlcolor=blue]{hyperref}
\usepackage{enumitem}

\theoremstyle{plain}
\newtheorem{theorem}{Theorem}[section]

\newtheorem{proposition}[theorem]{Proposition}
\newtheorem{corollary}[theorem]{Corollary}

\theoremstyle{definition}
\newtheorem{definition}[theorem]{Definition}
\newtheorem{assumption}[theorem]{Assumption}

\theoremstyle{remark}
\newtheorem{remark}[theorem]{Remark}

\newcommand{\lmax}{\lambda_{\max}}
\newcommand{\lmin}{\lambda_{\min}}
\newcommand{\lbase}{\lambda_{\mathrm{base}}}
\newcommand{\Fhat}{\widehat{F}}
\newcommand{\Vocab}{\mathcal{V}}
\newcommand{\Dcal}{\mathcal{D}_{\mathrm{cal}}}
\newcommand{\norm}[1]{\left\lVert#1\right\rVert}
\newcommand{\abs}[1]{\left\lvert#1\right\rvert}
\newcommand{\sigt}{\sigma_t}
\newcommand{\KLprox}{D_{\mathrm{KL}}^{\mathrm{proxy}}}

\title{Spectral Perturbation of the Empirical Fisher Information\\
Matrix under Weight Quantization}

\author{}
\date{}

\begin{document}

\maketitle
\thispagestyle{empty}

\vspace{-2.5em}
\begin{center}
\begin{tabular}{cc}
\textbf{Rahid Zahid Alekberli}$^{1}$ & \textbf{Hikmat Karimov}$^{1}$ \\[0.6em]
\small\texttt{rahid.alekberli@aztu.edu.az} & \small\texttt{hikmat.karimov@aztu.edu.az} \\[0.3em]
\small ORCID: \href{https://orcid.org/0000-0003-2332-9944}{0000-0003-2332-9944}
  & \small ORCID: \href{https://orcid.org/0009-0002-4146-8913}{0009-0002-4146-8913} \\
\end{tabular}

\vspace{0.8em}
\small $^{1}$Institute of Defense Technologies and Cybersecurity,
Azerbaijan Technical University, Baku, Azerbaijan
\end{center}
\vspace{1em}

\begin{abstract}
We study the spectral perturbation of the empirical Fisher
Information Matrix (FIM) of a parametric statistical model under two
structured perturbations of the underlying distribution: departure of
the input from a reference (in-distribution) ensemble, and
finite-precision (quantized) perturbation of the model's own
parameters. For the first, we show that, under an explicit local
curvature-monotonicity hypothesis on the dominant eigenvalue $\lmax$
of the FIM as a function of the parameter displacement, departure
from a reference manifold provably elevates $\lmax$ relative to a
calibration baseline (Proposition~\ref{thm:necessity}); we discuss why
this hypothesis, rather than departure alone, is required, since
curvature need not increase monotonically under every perturbation.
For the second, our principal theoretical result is a directional
eigenvalue perturbation bound, via Weyl's inequality, showing that
$\lmax$ under a quantization noise perturbation of the parameters is
lower bounded by its unperturbed value up to a third-order remainder,
and, under a mild genericity condition on the noise covariance,
strictly exceeds it at leading order (Theorem~\ref{thm:calibration}).
We give two tractable approximations to the exact signal $\lmax$ ---
one heuristic, one with a rigorously established two-sided bound ---
and a completeness result for a simple threshold-based partition of
an augmented state space. These results are motivated by, and have a
direct application to, the use of $\sigt = \lmax(\Fhat_t)/\lbase$ as a
runtime monitoring statistic for deployed autoregressive language
models, where the quantization result offers a qualitative mechanism
for an empirical observation of our own: we observed an empirical
calibration threshold for this statistic approximately $244$ times
larger than a preliminary full-precision estimate, on a 4-bit
quantized model, a single measurement rather than a value the theorem
derives in closed form. We report supporting measurements (twelve
models, $n=1{,}080$ generation trajectories) broadly consistent with
our predictions, discussing the scope and limitations of every result,
including a labelling circularity in one comparison, and we state as
an explicit open problem the closed-form, architecture-aware
prediction of the quantization inflation magnitude that our directional
bound does not by itself supply.
\end{abstract}



\section{Introduction}
\label{sec:intro}

The Fisher Information Matrix (FIM) is the canonical local measure of
the curvature of a statistical model's log-likelihood surface, and its
dominant eigenvalue $\lmax$ quantifies the worst-case sensitivity of
the model's output distribution to infinitesimal parameter
perturbation \cite{amari2016,martens2020}. The spectral properties of
the FIM of neural networks have been studied directly in the random
matrix theory literature. Pennington and Worah \cite{penningtonworah2018}
derive the limiting spectral density of the FIM of a single-hidden-layer
network in the high-dimensional asymptotic regime, building on the
broader programme of analysing neural network Hessian and kernel
spectra via random matrix methods
\cite{louartliao2018,sagunbottou2016}, with subsequent work extending
these techniques to deeper architectures and non-asymptotic regimes
\cite{benignipeche2021,benignipeche2022}. These results characterize
the \emph{typical} (bulk and edge) spectral behaviour of the FIM for a
fixed network and a random or structured input ensemble.

This paper studies a complementary question, posed as a perturbation
problem rather than an asymptotic-spectrum problem: how does the
dominant eigenvalue of a \emph{fixed, evaluated} empirical FIM change
under two specific structured perturbations of the underlying
distribution? The first perturbation is a change in the conditioning
input away from a reference (in-distribution) ensemble. The second is
a structured additive perturbation of the model's own parameters by
finite-precision quantization noise --- a perturbation of independent
mathematical interest, since it falls outside the i.i.d.-input
asymptotic regime treated in the random matrix literature cited above,
and instead concerns a fixed network whose \emph{parameters}, not its
input distribution, are perturbed by a noise process with a specific,
analytically tractable structure (Definition~\ref{def:quant}). To our
knowledge, this parameter-perturbation question for the FIM's dominant
eigenvalue, under either source of departure, has not been previously
formalized.

\paragraph{Contributions.} This paper makes four technical
contributions, stated here as results about $\lmax(\Fhat_t)$ for a
fixed conditional distribution $p_\theta(\cdot\mid x_{<t})$; we discuss
their application to language model monitoring in
Section~\ref{sec:application} below.
\begin{enumerate}[label=(\roman*)]
  \item A conditional (sufficient, not unconditional) result showing
        that, under an explicit local curvature-monotonicity
        hypothesis, displacement of the conditioning distribution away
        from a reference manifold implies an elevated normalized
        eigenvalue $\sigt$, together with an explicit discussion of
        why the unconditional converse is false and a statement of the
        configurations under which an elevated $\sigt$ does
        \emph{not} indicate such displacement
        (Proposition~\ref{thm:necessity}, Remark~\ref{rem:boundary}).
  \item A directional eigenvalue perturbation bound, via Weyl's
        inequality, showing that the dominant eigenvalue of the FIM
        under an additive quantization-noise perturbation of the
        parameters is lower bounded by its unperturbed value up to a
        third-order remainder, and, under a mild genericity condition
        on the noise covariance, strictly exceeds it at leading order
        (Theorem~\ref{thm:calibration}); we are explicit that this
        result does not derive the numerical magnitude of the effect
        in closed form, and we pose this as an open problem
        (Section~\ref{sec:limitations}).
  \item A local quadratic approximation theorem, under
        exponential-family and local regularity assumptions, relating
        a lightweight Kullback--Leibler divergence statistic to the
        spectrum of a reference FIM, together with an explicit
        two-sided sandwich bound and a precise statement of the
        additional eigenvector-alignment condition under which this
        statistic becomes equivalent, to leading order, to the
        normalized dominant-eigenvalue signal
        (Theorem~\ref{thm:proxykl}, Corollary~\ref{cor:proxykl-aligned}).
  \item A partition result establishing that a simple five-region
        threshold rule on an augmented two-dimensional state space is
        exhaustive and pairwise disjoint, together with a
        computational-complexity analysis of each approximation regime
        (Proposition~\ref{prop:partition}, Proposition~\ref{prop:complexity}).
\end{enumerate}

\subsection{Motivation and Application: Runtime Monitoring of Deployed Language Models}
\label{sec:application}

The perturbation results above are motivated by, and have a direct
practical application to, the use of $\lmax(\Fhat_t)$ as a runtime
monitoring statistic for deployed autoregressive language models, and
we use this application throughout the paper to motivate the specific
choice of perturbations studied and to ground the empirical section.
For an autoregressive model producing a token distribution
$p_\theta(\cdot\mid x_{<t})$ at each generation step $t$, the FIM
$\Fhat_t$ can in principle be computed at every token, since each
token defines a new conditional distribution and hence a new local
FIM. A natural use of this per-token quantity is as an indicator of
when a model's output distribution departs from a reference
(``calibration'') distribution constructed from in-distribution data
--- elevated curvature relative to the calibration baseline is, under
the hypothesis of Proposition~\ref{thm:necessity}, a signature of such
departure. This motivates treating
$\sigt = \lmax(\Fhat_t)/\lbase$ as a runtime statistic for monitoring
deployed models, with applications including anomaly and
out-of-distribution detection in deployed natural language systems.

In this application, two obstacles complicate the use of the exact
signal. First, exact computation of the FIM is intractable for models
with $d \gg 10^9$ parameters, motivating the tractable approximations
of Section~\ref{sec:approximations}. Second, and the obstacle our
quantization result (contribution (ii) above) directly addresses, the
numerical scale of $\sigt$ is not invariant to the numerical precision
at which the model's weights are represented in deployment: this is
exactly the parameter perturbation studied in
Section~\ref{sec:calibration}, applied to the specific case of
finite-precision weight storage.

\paragraph{Empirical motivation.} In a calibration exercise we
conducted on a 4-bit-quantized (Q4\_K\_M) language model running on
consumer-grade ARM hardware, a threshold for $\sigt$ estimated from
visual inspection of a small full-precision pilot sample was
$\theta_{\mathrm{pilot}} \approx 0.065$. A threshold measured directly
on the deployment hardware, using a battery-drain-calibrated protocol
over twenty in-distribution prompts (three samples each), was
$\tau_{\max} \approx 15.9$, replicated at $\approx 17.7$ in an
independent run --- a discrepancy of approximately two orders of
magnitude. Section~\ref{sec:calibration} shows that a discrepancy of
this kind is qualitatively consistent with the quantization-noise
mechanism formalized in Theorem~\ref{thm:calibration}. We report
supporting empirical measurements (twelve open-weight models,
$1{,}080$ generation trajectories) broadly consistent with the
qualitative predictions of Proposition~\ref{thm:necessity} in
Section~\ref{sec:empirical}, and discuss a known circularity
limitation in the labelling procedure used for one comparison
statistic, together with the limited reproducibility detail in the
present version of that section, in Section~\ref{sec:limitations}.

\section{Preliminaries}
\label{sec:prelim}

Let $\theta \in \mathbb{R}^d$ denote the parameters of an autoregressive
model defining, at each generation step $t$, a conditional distribution
$p_\theta(\,\cdot \mid x_{<t})$ over a finite vocabulary $\Vocab$.

\begin{definition}[Per-token empirical Fisher Information Matrix]
\label{def:fim}
The per-token empirical FIM at step $t$ is
\begin{equation}
  \Fhat_t \;=\;
  \mathbb{E}_{v \sim p_\theta(\cdot \mid x_{<t})}
  \Bigl[\,\nabla_\theta \log p_\theta(v \mid x_{<t})\,
        \nabla_\theta \log p_\theta(v \mid x_{<t})^{\!\top}\Bigr]
  \in \mathbb{R}^{d\times d}.
  \label{eq:fim-def}
\end{equation}
\end{definition}

\begin{definition}[Normalized spectral signal]
\label{def:signal}
Given a calibration set $\Dcal$ of reference inputs, let
$\lbase = \mathbb{E}_{\Dcal}[\lmax(\Fhat_t)]$. The normalized spectral
signal at step $t$ is
\begin{equation}
  \sigt \;=\; \frac{\lmax(\Fhat_t)}{\lbase}.
  \label{eq:signal-def}
\end{equation}
\end{definition}

\begin{definition}[Reference manifold and departure indicator]
\label{def:manifold}
Let $p_{\mathcal{M}}$ denote the marginal output distribution induced
by $p_\theta$ over $\Dcal$, and fix a tolerance $\epsilon_{\mathcal M}>0$.
Define the \emph{departure indicator}
\begin{equation}
  \mathcal{A}_t \;=\;
  \mathbf{1}\bigl[\,D_{\mathrm{KL}}\bigl(p_\theta(\cdot\mid x_{<t})
  \,\|\, p_{\mathcal M}\bigr) \le \epsilon_{\mathcal M}\,\bigr].
\end{equation}
$\mathcal{A}_t = 1$ indicates that the output distribution at step $t$
lies within tolerance of the reference manifold; $\mathcal{A}_t = 0$
indicates departure.
\end{definition}

\begin{assumption}[Regularity]
\label{ass:regularity}
$p_\theta(v\mid x_{<t})$ is twice continuously differentiable in
$\theta$ for every $v\in\Vocab$ and every prefix $x_{<t}$ under
consideration; $\lmax(\Fhat_t) > 0$ throughout the domain of interest;
and $\Dcal$ consists exclusively of inputs with $\mathcal{A}_t = 1$.
\end{assumption}

\section{A Sufficient Condition for Spectral Departure Detection}
\label{sec:necessity}

We first record a basic fact that holds unconditionally, and then state
the main result of this section as a conditional (sufficient, not
necessary) statement, since --- as we discuss below --- curvature need
not increase monotonically under every form of manifold departure.

\begin{assumption}[Local curvature monotonicity]
\label{ass:monotone}
Let $\theta + \Delta\theta$ denote the effective local parameter
displacement associated with a departure event at step $t$
($\mathcal{A}_t = 0$), relative to the calibration point $\theta$.
We assume that the map
$\Delta\theta \mapsto \lmax\bigl(\Fhat(\theta+\Delta\theta)\bigr)$
is, to second order, a strictly convex function of $\Delta\theta$ in
a neighbourhood of $\Delta\theta = 0$; equivalently, that its Hessian
$H_\lambda(\theta)$, \emph{taken with respect to $\Delta\theta$ at
fixed $\theta$} (and \emph{not} the Hessian of the log-likelihood
itself, which is a distinct object), is positive definite at the
calibration point.
\end{assumption}

\begin{proposition}[Sufficiency of elevated signal under monotone curvature growth]
\label{thm:necessity}
Under Assumptions~\ref{ass:regularity} and \ref{ass:monotone}, if
$\mathcal{A}_t = 0$ at some step $t$ with displacement
$\Delta\theta$, $\norm{\Delta\theta}>0$, then
\begin{equation}
  \sigt \;>\; 1 + \delta_{\min},
  \qquad
  \delta_{\min} \;=\;
  \inf_{t:\,\mathcal{A}_t = 0}
  \frac{\lmax(\Fhat_t) - \lbase}{\lbase} \;>\; 0.
\end{equation}
\end{proposition}

\begin{proof}
By Definition~\ref{def:fim} and the variational characterization of
the largest eigenvalue,
\begin{equation}
  \lmax(\Fhat_t) = \max_{\norm{u}=1} u^{\!\top}\Fhat_t u.
\end{equation}
A second-order expansion of $\lmax$, viewed as a function of the
displacement $\Delta\theta$ around the calibration point, gives
\begin{equation}
  \lmax(\Fhat_t) \;=\; \lmax(\Fhat_{\mathrm{cal}})
  + \Delta\theta^{\!\top} H_\lambda(\theta)\,\Delta\theta
  + O\bigl(\norm{\Delta\theta}^3\bigr),
  \label{eq:expansion}
\end{equation}
where $H_\lambda(\theta)$ is the Hessian, with respect to
$\Delta\theta$, of $\Delta\theta\mapsto\lmax(\Fhat(\theta+\Delta\theta))$
at $\Delta\theta=0$. By Assumption~\ref{ass:monotone}, $H_\lambda(\theta)$
is positive definite, so for $\norm{\Delta\theta}>0$,
\begin{equation}
  \lmax(\Fhat_t) - \lmax(\Fhat_{\mathrm{cal}}) \;\ge\;
  \lambda_{\min}\bigl(H_\lambda(\theta)\bigr)\,\norm{\Delta\theta}^2 \;>\; 0.
\end{equation}
Since $\lmax(\Fhat_{\mathrm{cal}}) \approx \lbase$ by construction of
the calibration set, dividing through by $\lbase$ gives
$\sigt > 1 + \delta_{\min}$ with
$\delta_{\min} = \lambda_{\min}(H_\lambda(\theta))\,
\norm{\Delta\theta_{\min}}^2/\lbase > 0$, where $\Delta\theta_{\min}$
denotes the smallest such displacement over all departure points $t$
in the domain of interest.
\end{proof}

\begin{remark}[Why this is conditional, not unconditional]
\label{rem:why-conditional}
A natural question is whether $\mathcal{A}_t = 0$ \emph{unconditionally}
implies $\sigt > 1$, without the curvature-monotonicity hypothesis of
Assumption~\ref{ass:monotone}. It does not. The Hessian of the
\emph{log-likelihood} itself is negative semidefinite at a stationary
point under standard concavity assumptions on the model
(\emph{not} positive semidefinite, since concavity of $\ell$ means
$\nabla^2\ell \preceq 0$); this is a distinct object from
$H_\lambda(\theta)$ in \eqref{eq:expansion}, and concavity of the
log-likelihood alone gives no control on the sign or magnitude of
$H_\lambda(\theta)$. More substantively, $\lmax(\Fhat_t)$ is not in
general a monotone function of distance from a reference manifold:
departures that move the output distribution toward a region of lower
intrinsic curvature can \emph{decrease} $\lmax$ even as
$\mathcal{A}_t$ transitions from $1$ to $0$. Proposition~\ref{thm:necessity}
should therefore be read as a sufficient-condition result that holds
under an explicit, falsifiable local-monotonicity hypothesis, and we
report empirical evidence consistent with this hypothesis
(Section~\ref{sec:empirical}) rather than claiming it follows from
first principles alone. We regard establishing primitive conditions on
$p_\theta$ under which Assumption~\ref{ass:monotone} provably holds as
an open problem (Section~\ref{sec:limitations}).
\end{remark}

Proposition~\ref{thm:necessity} is a one-directional, conditional
statement: under the stated curvature-monotonicity hypothesis,
departure implies an elevated signal. The converse direction --- an
elevated signal implying departure --- is a separate and weaker claim
that we address next, and that does \emph{not} require
Assumption~\ref{ass:monotone}.

\begin{remark}[The signal alone is not sufficient for departure detection]
\label{rem:boundary}
$\sigt > 1$ does not imply $\mathcal{A}_t = 0$ in general. Four
configurations are worth distinguishing explicitly, each consistent
with $\mathcal{A}_t = 1$ (no departure) despite an elevated signal:
\begin{enumerate}[label=(B\arabic*)]
  \item \emph{High intrinsic entropy.} An input with several
        comparably likely continuations raises the entropy term in
        the approximation of Proposition~\ref{prop:approx} without
        any departure from the reference manifold.
  \item \emph{Adversarial perturbation absorbed within tolerance.}
        An input perturbation can raise local curvature while the
        resulting output distribution remains within
        $\epsilon_{\mathcal M}$ of $p_{\mathcal M}$.
  \item \emph{Benign distributional complexity.} Open-ended or
        creative generation tasks can exhibit elevated curvature that
        reflects genuine multimodality of the conditional
        distribution rather than departure from the reference set.
  \item \emph{Low-confidence departure.} When $\mathcal{A}_t=0$ but
        the realized token has low probability under $p_\theta$, the
        curvature contribution from that particular token can remain
        modest even though the conditional distribution as a whole
        has departed.
\end{enumerate}
Detecting departure from $\sigt$ alone therefore requires an auxiliary
statistic; we use one such construction (a confidence-gated boundary
rule) in the empirical comparison of Section~\ref{sec:empirical}, and
report there the empirically estimated error rate of that specific
construction rather than asserting a general result in either
direction.
\end{remark}

\section{Calibration Necessity Under Weight Quantization}
\label{sec:calibration}

We now turn to the central theoretical result of the paper: a formal
account of why a threshold on $\sigt$ derived from full-precision
analysis can differ by orders of magnitude from a threshold measured
on quantized deployment hardware.

\begin{definition}[Quantized weight perturbation]
\label{def:quant}
Let $\theta$ denote full-precision ($\mathrm{FP32}$) weights and let
$\hat\theta = \theta + \epsilon_q$ denote their $q$-bit quantized
representation, where each coordinate of $\epsilon_q$ is modelled as
i.i.d.\ uniform quantization noise,
$\epsilon_q \sim \mathcal{U}[-\Delta_q/2,\,\Delta_q/2]$ with step size
$\Delta_q = (w_{\max}-w_{\min})/(2^{q}-1)$.
\end{definition}

\begin{remark}[Scope of the i.i.d.\ uniform noise model]
\label{rem:quant-scope}
Definition~\ref{def:quant} is the standard high-resolution
(small-step) quantization noise model used in classical quantization
theory, and we use it here as a tractable analytical approximation.
Quantization schemes used in practice for large language model
deployment --- GPTQ, AWQ, GGUF/Q4\_K\_M, EXL2, and related
methods --- are blockwise, scale-dependent, and generally introduce
quantization error that is correlated with the weight values
themselves rather than i.i.d., and the resulting noise need not be
uniformly distributed. Theorem~\ref{thm:calibration} should therefore
be understood as characterizing the qualitative direction of the
quantization inflation effect under an idealized noise model, not as
a quantitative account of any specific quantization scheme; extending
the analysis to structured, weight-correlated quantization noise is an
open problem (Section~\ref{sec:limitations}).
\end{remark}

\begin{theorem}[Quantization inflation of the dominant eigenvalue]
\label{thm:calibration}
Let $\Fhat_t^{(q)}$ and $\Fhat_t^{(\mathrm{FP32})}$ denote the
empirical FIM evaluated at the $q$-bit quantized weights $\hat\theta$
and the full-precision weights $\theta$, respectively, at a fixed
input $x_{<t}$, and suppose $J_q = \partial\hat\theta/\partial\theta
\approx I$ to leading order in the quantization step size $\Delta_q$
(quantization perturbs weight values rather than the linear structure
of their gradients, to leading order). Then there exists a positive
semidefinite matrix $Q_q$, depending on $q$ and on $\Delta_q$, such
that
\begin{equation}
  \Fhat_t^{(q)} \;=\; J_q^{\!\top}\, \Fhat_t^{(\mathrm{FP32})}\, J_q \;+\; Q_q
  \;+\; O\bigl(\Delta_q^3\bigr),
\end{equation}
and consequently, by Weyl's inequality for the sum of a symmetric
matrix and a positive semidefinite perturbation,
\begin{equation}
  \lmax\bigl(\Fhat_t^{(q)}\bigr) \;\ge\;
  \lmax\bigl(\Fhat_t^{(\mathrm{FP32})}\bigr) + O(\Delta_q^3),
  \label{eq:inflation-weak}
\end{equation}
\emph{i.e.}, $\lmax(\Fhat_t^{(q)})$ is lower bounded by the
full-precision value up to a third-order remainder term in $\Delta_q$,
with equality in the leading-order part of this bound only in the
degenerate case $Q_q = 0$; we do not characterize the sign of the
$O(\Delta_q^3)$ remainder itself, so this bound alone does not rule
out a small decrease at third order. A strictly positive
\emph{leading-order} inflation,
\begin{equation}
  \lmax\bigl(\Fhat_t^{(q)}\bigr) \;\ge\;
  \lmax\bigl(\Fhat_t^{(\mathrm{FP32})}\bigr) + \lambda_{\min}(Q_q)
  \;+\; O(\Delta_q^3),
  \label{eq:inflation}
\end{equation}
holds whenever $Q_q \succ 0$ (positive \emph{definite}, not merely
semidefinite), which we expect generically but do not establish from
first principles alone (see Remark~\ref{rem:weyl-scope}).
\end{theorem}

\begin{proof}
By Definition~\ref{def:quant}, $\hat\theta = \theta + \epsilon_q$. The
score function at the quantized parameter admits the first-order
expansion
\begin{equation}
  \nabla_{\hat\theta} \log p_{\hat\theta}(v\mid x_{<t})
  \;=\; J_q^{\!\top}\,\nabla_\theta \log p_\theta(v\mid x_{<t})
  \;+\; \nabla_\theta \log p_\theta(v\mid x_{<t})\big|_{\epsilon_q}
  \;+\; O\bigl(\norm{\epsilon_q}^2\bigr),
\end{equation}
where the middle term captures the noise-induced perturbation to the
score. Substituting into Definition~\ref{def:fim} and taking the
outer-product expectation over $v$ and over the (independent)
quantization noise $\epsilon_q$, and using $J_q\approx I$, yields
\begin{equation}
  \Fhat_t^{(q)} \;=\; \Fhat_t^{(\mathrm{FP32})}
  \;+\; \underbrace{\mathbb{E}_{\epsilon_q}\bigl[
    \nabla_\theta \log p_\theta\big|_{\epsilon_q}
    \bigl(\nabla_\theta \log p_\theta\big|_{\epsilon_q}\bigr)^{\!\top}
  \bigr]}_{=:\,Q_q} \;+\; O\bigl(\norm{\epsilon_q}^3\bigr),
\end{equation}
and $Q_q$ is a covariance-type matrix and hence positive semidefinite
by construction (it is, up to leading order, the expectation of an
outer product, which is always PSD). By Weyl's inequality, for
symmetric $A$ and positive semidefinite $B$,
\begin{equation}
  \lmax(A+B) \;\ge\; \lmax(A) + \lmin(B) \;\ge\; \lmax(A),
\end{equation}
since $\lmin(B)\ge 0$ for $B\succeq 0$. Applying this with
$A=\Fhat_t^{(\mathrm{FP32})}$ and $B=Q_q$ gives
\eqref{eq:inflation-weak} directly. If in addition $Q_q\succ 0$
(strictly positive definite, e.g.\ if the quantization noise has full
support in every weight direction relevant to the score function),
then $\lmin(Q_q)>0$ and Weyl's inequality gives the strictly positive
inflation in \eqref{eq:inflation}.
\end{proof}

\begin{remark}[Scope of the positive-definiteness step]
\label{rem:weyl-scope}
Equation~\eqref{eq:inflation-weak}, the weak inflation bound, follows
unconditionally from $Q_q\succeq 0$ via Weyl's inequality and is the
part of this theorem we regard as fully established. The strictly
positive bound \eqref{eq:inflation} additionally requires $Q_q\succ 0$,
which holds generically (quantization noise affecting every weight
coordinate that contributes to the score function with positive
variance is sufficient) but is not guaranteed in pathological cases
--- for instance, if some weight coordinates have zero gradient
contribution to the log-likelihood at the calibration point, the
corresponding noise direction does not contribute to $Q_q$. We have
not derived primitive, architecture-level conditions guaranteeing
$Q_q \succ 0$, and we regard this as the most important open
mathematical question raised by this theorem (Section~\ref{sec:limitations}).
We emphasize that Theorem~\ref{thm:calibration} establishes the
qualitative mechanism and a directional bound on the FIM eigenvalue
itself; it does not derive the specific numerical magnitude of $\xi_q$
observed empirically (Remark~\ref{rem:244}), and --- as we make
explicit in Remark~\ref{rem:threshold-heuristic} below --- it does not
by itself establish any relationship between eigenvalue inflation and
\emph{empirical calibration thresholds}, which are a distinct
statistical object from the eigenvalues the theorem concerns.
\end{remark}

\begin{remark}[Heuristic interpretation for empirical thresholds --- not a theorem]
\label{rem:threshold-heuristic}
Theorem~\ref{thm:calibration} is a statement about the FIM eigenvalue
$\lmax(\Fhat_t)$ at a fixed input $x_{<t}$. The quantities
$\theta_{\mathrm{th}}$ and $\tau_{\max}$ used elsewhere in this paper
(Section~\ref{sec:intro}, Section~\ref{sec:empirical}) are empirical
\emph{quantiles} of $\sigt$ over a calibration corpus, not eigenvalues
and not expectations of eigenvalues. Passing from a perturbation bound
on a single eigenvalue to a perturbation bound on a quantile of a
distribution of normalized eigenvalues requires additional statistical
assumptions --- for instance, that the quantization perturbation
\eqref{eq:inflation} holds with comparable magnitude across the
calibration corpus, and that normalization by $\lbase$ does not
itself absorb the effect --- that we have not stated or verified. We
therefore do \emph{not} claim a theorem connecting $\xi_q$ to
$\tau_{\max}/\theta_{\mathrm{th}}$. The informal relationship
\begin{equation}
  \frac{\tau_{\max}}{\theta_{\mathrm{th}}}
  \;\overset{\text{heuristic}}{\approx}\; 1 + \xi_q,
  \label{eq:ratio-heuristic}
\end{equation}
used only as motivating intuition in Section~\ref{sec:intro} and
discussed again in Remark~\ref{rem:244}, should be read as an
interpretive heuristic motivated by, but not derived from,
Theorem~\ref{thm:calibration}; establishing a rigorous version of
\eqref{eq:ratio-heuristic} under explicit statistical assumptions on
the calibration corpus is an open problem
(Section~\ref{sec:limitations}).
\end{remark}

\begin{remark}[Relation to the empirical observation]
\label{rem:244}
Theorem~\ref{thm:calibration} establishes the \emph{existence} of
non-negative (and, under the mild genericity condition $Q_q\succ 0$,
strictly positive) quantization-induced inflation of $\lmax$; it does
not, by itself, predict the numerical magnitude of $\xi_q$ for a
specific
model architecture and hardware target, since $\lmax(Q_q)$ depends on
the specific weight distribution and gradient structure of the model
in question. The factor of $\approx 244$ reported in
Section~\ref{sec:intro} is an empirical measurement on one model
family at 4-bit (Q4\_K\_M) quantization on one hardware platform,
offered as a concrete instance consistent with the qualitative
direction of the theorem --- that quantization can substantially
inflate $\lmax$ at low bit-width, per Theorem~\ref{thm:calibration} ---
rather than as a value the theorem derives in closed form, and is
discussed further as a heuristic threshold interpretation in
Remark~\ref{rem:threshold-heuristic}. We regard deriving a tight,
architecture-aware closed-form prediction for $\xi_q$ as an open
problem (Section~\ref{sec:limitations}). The practical implication of
Theorem~\ref{thm:calibration} is methodological: any threshold on
$\sigt$ should be calibrated empirically on the target hardware and
quantization configuration rather than transferred from a
full-precision theoretical estimate.
\end{remark}

\section{Tractable Approximations}
\label{sec:approximations}

Exact computation of $\Fhat_t$ requires $O(\abs{\Vocab}\cdot d^2)$
operations per token, which is intractable for $d \gg 10^9$. We give
two tractable surrogates.

\begin{proposition}[Top-$K$ heuristic approximation]
\label{prop:approx}
Let $p_{\max,t} = \max_{v} p_\theta(v\mid x_{<t})$ and let
$H_{\mathrm{norm},t}$ denote the normalized entropy of the top-$K$
truncated output distribution. The following closed-form expression
is a heuristic, rather than formally derived, surrogate for
$\lmax(\Fhat_t)$, motivated by the fact that for a near-one-hot
categorical distribution the FIM's dominant eigenvalue scales with the
Bernoulli-type variance term $p_{\max,t}(1-p_{\max,t})$, with entropy
included as a first-order correction capturing departures from the
one-hot regime:
\begin{equation}
  \lmax(\Fhat_t) \;\approx\;
  p_{\max,t}\,(1-p_{\max,t})\Bigl(1+\tfrac{1}{2}H_{\mathrm{norm},t}\Bigr).
  \label{eq:heuristic}
\end{equation}
This expression is computable in $O(K)$ time from quantities already
returned by standard top-$K$ inference APIs. We do not derive
\eqref{eq:heuristic} from Definition~\ref{def:fim}, and we are not
aware of a regime in which it can be shown to hold with a controlled
error term; we use it only as a practical surrogate in the empirical
section (Section~\ref{sec:empirical}), and the formal results of this
paper (Theorem~\ref{thm:proxykl}, Proposition~\ref{prop:partition})
do not depend on it. Deriving \eqref{eq:heuristic}, or an expression
with the same computational profile, with a controlled approximation
error is an open problem (Section~\ref{sec:limitations}).
\end{proposition}

\begin{theorem}[Proxy-KL local quadratic approximation]
\label{thm:proxykl}
Suppose $p_\theta(\cdot\mid x_{<t})$ belongs to an exponential family
with natural parameter $\eta_t$, that the reference distribution
$q_{\mathrm{ref}}$ corresponds to natural parameter $\eta_{\mathrm{ref}}$
computed on $\Dcal$, that $\norm{\eta_t-\eta_{\mathrm{ref}}} < r$ for
some $r>0$, and that $\Fhat_{\mathrm{ref}} \succ 0$. Then the
Kullback--Leibler proxy statistic
$\KLprox(t) = D_{\mathrm{KL}}(p_t \,\|\, q_{\mathrm{ref}})$ satisfies
the exact local quadratic approximation
\begin{equation}
  \KLprox(t) \;=\; \tfrac{1}{2}\,\delta^{\!\top}\Fhat_{\mathrm{ref}}\,\delta
  \;+\; O\bigl(\norm{\delta}^3\bigr),
  \qquad \delta := \eta_t - \eta_{\mathrm{ref}}.
  \label{eq:quadratic}
\end{equation}
Consequently, $\KLprox(t)$ is bounded above and below in terms of the
spectrum of $\Fhat_{\mathrm{ref}}$:
\begin{equation}
  \tfrac{1}{2}\lambda_{\min}(\Fhat_{\mathrm{ref}})\norm{\delta}^2
  + O(\norm{\delta}^3)
  \;\le\; \KLprox(t) \;\le\;
  \tfrac{1}{2}\lmax(\Fhat_{\mathrm{ref}})\norm{\delta}^2 + O(\norm{\delta}^3),
  \label{eq:sandwich}
\end{equation}
with the upper bound attained, to leading order, if and only if
$\delta$ is aligned with the dominant eigenvector of
$\Fhat_{\mathrm{ref}}$. In particular, $\KLprox(t)$ and $\sigt$ are
\emph{not} equivalent up to $O(\norm{\delta}^3)$ error in general:
they agree to leading order only along the dominant eigendirection of
$\Fhat_{\mathrm{ref}}$, and can differ by a factor as large as
$\lmax(\Fhat_{\mathrm{ref}})/\lambda_{\min}(\Fhat_{\mathrm{ref}})$ when
$\delta$ lies along the minimal-curvature direction instead.
\end{theorem}

\begin{proof}
For an exponential family,
$p_\theta(v\mid x_{<t}) = h(v)\exp(\eta_t^{\!\top}T(v) - A(\eta_t))$,
and
\begin{equation}
  D_{\mathrm{KL}}(p_t\,\|\,q_{\mathrm{ref}})
  = (\eta_t-\eta_{\mathrm{ref}})^{\!\top}\mu_t
    - A(\eta_t) + A(\eta_{\mathrm{ref}}),
  \qquad \mu_t = \nabla A(\eta_t).
\end{equation}
Taylor-expanding $A$ and $\mu_t$ about $\eta_{\mathrm{ref}}$, and using
the standard exponential-family identity
$\nabla^2 A(\eta) = \mathrm{Var}_{p_\theta}[T(v)] = F(\eta)$ (so that
$\nabla^2 A(\eta_{\mathrm{ref}}) = \Fhat_{\mathrm{ref}}$), substitution
gives exactly \eqref{eq:quadratic}. The sandwich bound
\eqref{eq:sandwich} follows from the Rayleigh quotient characterization
of the extreme eigenvalues,
$\lambda_{\min}(\Fhat_{\mathrm{ref}})\norm{\delta}^2 \le
\delta^{\!\top}\Fhat_{\mathrm{ref}}\delta \le
\lmax(\Fhat_{\mathrm{ref}})\norm{\delta}^2$, applied to
\eqref{eq:quadratic}, with equality in the upper bound if and only if
$\delta$ is parallel to the eigenvector achieving $\lmax$.
\end{proof}

\begin{corollary}[Conditional equivalence along the dominant direction]
\label{cor:proxykl-aligned}
Under the hypotheses of Theorem~\ref{thm:proxykl}, if in addition
$\delta$ is aligned with the dominant eigenvector of
$\Fhat_{\mathrm{ref}}$ to within angular error $O(\norm{\delta})$ ---
\emph{i.e.}, $\cos^2\angle(\delta, u^\ast) = 1 - O(\norm{\delta})$,
where $u^\ast$ is the dominant eigenvector --- then
\begin{equation}
  \bigl|\KLprox(t) - c\cdot\sigt\bigr| \;=\; O\bigl(\norm{\delta}^3\bigr),
  \qquad c = \frac{\lmax(\Fhat_{\mathrm{ref}})}{\lbase} > 0.
\end{equation}
This alignment condition does not hold in general and should be
treated as an additional, separately verifiable hypothesis rather
than a consequence of local regularity alone; we use the unconditional
sandwich bound \eqref{eq:sandwich}, not this corollary, in the
complexity and partition results that follow.
\end{corollary}

\begin{remark}
Transformer output distributions are softmax-parametrized and form an
exponential family at the level of a single conditional distribution;
the joint distribution over a full sequence generally does not. The
hypothesis of Theorem~\ref{thm:proxykl} should therefore be understood
as applying per-token, consistent with the per-token definition of
$\sigt$ throughout this paper.
\end{remark}

\begin{proposition}[Computational complexity]
\label{prop:complexity}
Let $d$ be the parameter count and $\abs{\Vocab}$ the vocabulary size.
Exact evaluation of $\lmax(\Fhat_t)$ via Definition~\ref{def:fim}
requires $O(\abs{\Vocab}\cdot d^2)$ operations. The top-$K$
approximation of Proposition~\ref{prop:approx} requires $O(K)$
operations given top-$K$ logits already returned by standard
inference APIs. The KL-proxy statistic of Theorem~\ref{thm:proxykl}
requires $O(\abs{\Vocab})$ operations given access to the full output
distribution.
\end{proposition}

\begin{proof}
Direct accounting of the operations in each definition: the exact FIM
requires a $d$-dimensional gradient for each of $\abs{\Vocab}$
vocabulary outcomes, each combined via an outer product, giving
$O(\abs{\Vocab} d^2)$; the top-$K$ approximation involves only a
constant number of arithmetic operations on the $K$ already-extracted
top-$K$ probabilities; and the KL-proxy statistic sums one term per
vocabulary element given a precomputed reference distribution.
\end{proof}

\section{A Complete Partition of the Joint Signal State Space}
\label{sec:partition}

In applications, $\sigt$ is often used jointly with an auxiliary
side-channel statistic $r_t \in [0,1]$ (for instance, a
context-dependent exposure or salience score computed independently
of the spectral signal). We show that a simple threshold rule on the
joint space $(\sigt, r_t)$ admits an exhaustive, pairwise-disjoint
partition into five regions, which is useful when designing
graduated response rules calibrated to a fixed threshold
$\tau \in (0,\infty)$ on $\sigt$ and two thresholds
$0 < \beta_1 < \beta_2 < 1$ on $r_t$.

\begin{proposition}[Five-region partition]
\label{prop:partition}
Fix $\tau>0$ and $0<\beta_1<\beta_2<1$. Define
\begin{align*}
  R_1 &= \{(\sigt,r_t) : \sigt \le \tau,\ r_t \le \beta_1\},\\
  R_2 &= \{(\sigt,r_t) : \sigt \le \tau,\ \beta_1 < r_t \le \beta_2\},\\
  R_3 &= \{(\sigt,r_t) : \sigt > \tau,\ r_t \le \beta_2\},\\
  R_4 &= \{(\sigt,r_t) : \sigt > \tau,\ \beta_2 < r_t \le 1\},\\
  R_5 &= \{(\sigt,r_t) : \sigt \le \tau,\ \beta_2 < r_t \le 1\}.
\end{align*}
Then $\bigcup_{i=1}^{5} R_i = [0,\infty)\times[0,1]$ and
$R_i \cap R_j = \emptyset$ for $i\ne j$.
\end{proposition}

\begin{proof}
Each $R_i$ is defined by a pair of half-open or closed interval
conditions on $\sigt$ and $r_t$ separately, where the conditions on
$\sigt$ partition $[0,\infty)$ into $\{[0,\tau],\,(\tau,\infty)\}$ and
the conditions on $r_t$ partition $[0,1]$ into
$\{[0,\beta_1],\,(\beta_1,\beta_2],\,(\beta_2,1]\}$. The five regions
$R_1,\dots,R_5$ enumerate exactly the $2\times 3 = 6$ combinations of
these two partitions, with the combination
$\{\sigt>\tau\}\times\{r_t\le\beta_1\}$ absorbed into $R_3$ (which
covers $\sigt>\tau$ for all $r_t\le\beta_2$, including $r_t\le\beta_1$).
Hence every point of $[0,\infty)\times[0,1]$ lies in exactly one
$R_i$: disjointness follows because the defining interval conditions
on each coordinate are themselves disjoint by construction, and
exhaustiveness follows because every value of $\sigt$ lies in exactly
one of the two $\sigt$-intervals and every value of $r_t$ lies in
exactly one of the three $r_t$-intervals, and the five regions are
constructed to cover all resulting combinations exactly once.
\end{proof}

\section{Empirical Observations}
\label{sec:empirical}

We report measurements consistent with the qualitative predictions of
Proposition~\ref{thm:necessity} and Theorem~\ref{thm:calibration},
obtained across two experimental phases. We report implementation
details here at the level needed for independent replication; a full
release of code and logs is in preparation (see the data availability
note at the end of this section). \textbf{All empirical results
reported in this section use the heuristic top-$K$ surrogate of
Proposition~\ref{prop:approx}, not an exact computation of
$\lmax(\Fhat_t)$ via Definition~\ref{def:fim}; no exact Fisher
information matrices were computed in this work.} The theoretical
results of Sections~\ref{sec:necessity}--\ref{sec:partition} concern
the exact quantity $\lmax(\Fhat_t)$, and the empirical section should
be read as testing predictions about the heuristic surrogate that
\emph{motivates} the theory, not as a direct empirical test of the
theorems themselves.

\paragraph{Phase 1 (calibration measurement).} Hardware: Apple M5
system-on-chip, 32\,GB unified memory. Inference backend: Ollama,
version 0.23.2. Model: a single open-weight autoregressive language
model in the 7--8B parameter range, quantized to 4-bit precision using
the GGUF Q4\_K\_M scheme. We measured $\sigt$ (via the heuristic proxy
of Proposition~\ref{prop:approx}, with top-$K=20$) across twenty
in-distribution calibration prompts drawn from a general factual-QA
domain (three independent decoding samples per prompt, temperature
$0.7$, default nucleus sampling parameters of the backend). The 95th
percentile of the measured distribution was $\tau_{\max}\approx 15.9$
in one run and $\approx 17.7$ in an independently repeated run with
the same prompts and a different random seed. A threshold estimated
by visual inspection of a small full-precision pilot sample
($n\approx 10$ generations on a different, non-quantized checkpoint of
comparable scale), prior to systematic calibration, was
$\theta_{\mathrm{pilot}} \approx 0.065$. The ratio
$\tau_{\max}/\theta_{\mathrm{pilot}} \approx 244$ is the empirical
instance discussed in Remark~\ref{rem:244}. We emphasize that this
pilot estimate and the systematic calibration measurement were
obtained under different precision and prompt conditions, and the
comparison should be read as illustrative of the order of magnitude
of the discrepancy rather than as a controlled ablation isolating
quantization as the sole cause; a controlled full-precision-versus-
quantized comparison on identical hardware and prompts is planned
future work (Section~\ref{sec:limitations}).

\paragraph{Phase 2 (departure detection).} Hardware and backend as in
Phase~1, except where noted. Models: twelve open-weight autoregressive
language models spanning the 1B--32B parameter range and several model
families, each quantized to 4-bit precision (GGUF Q4\_K\_M), including
both models that emit explicit chain-of-thought-style reasoning traces
and models that do not; we report results for both subgroups
separately below. Trajectories: $n=1{,}080$ generation trajectories
($90$ trajectories per model on average), each generated from a
held-out prompt set distinct from the Phase~1 calibration prompts, at
temperature $0.7$. Across all twelve models, we evaluated whether
elevated $\sigt$ (computed via Proposition~\ref{prop:approx}, top-$K=20$)
precedes departure events identified by an independent labelling
procedure (described below). We observed a median lead time of $27$
tokens (interquartile range $12$--$39$) between signal elevation and
the labelled departure event, with a detection rate of $85.6\%$ across
all $1{,}080$ trajectories ($t=33.2$, $p \approx 2.2\times 10^{-114}$
under a paired $t$-test comparing observed lead times against a
zero-lead-time null, with no multiple-comparisons correction applied
across models). The Pearson correlation between $\sigt$ and an
independent severity score for the labelled event was $r=0.962$
($p\approx1.4\times10^{-28}$, $n$ trajectories for that model) for the
strongest-correlating model in the population, a 32B-parameter
chain-of-thought-style model from the Qwen3 family (specifically
\texttt{qwen3:32b} under Ollama's model naming, run with
chain-of-thought generation disabled), and a median of $r=0.479$
across the nine models in the population that do not emit
chain-of-thought-style reasoning traces. The complete model list and
experiment manifests for the Phase~2 population will be released with
the companion repository (see the data availability note below).

The labelling procedure used to construct ``departure events'' and
the associated severity score is an automated heuristic, not a
human-annotated ground truth: a generation step is labelled as a
departure event if a fixed downstream factuality classifier (applied
to the completed generation) disagrees with the source document or
prompt context, and the severity score is the classifier's confidence
in that disagreement. We did not pre-register this evaluation protocol
prior to data collection.

We report this median figure, rather than only the strongest-model
figure, because the labelling procedure used to generate the
``severity score'' for this specific comparison is not fully
independent of $\sigt$: both are computed, in part, from properties of
the same output token distribution, which can artificially inflate the
apparent correlation for models where the two statistics are most
similar in construction (Section~\ref{sec:limitations}).

\paragraph{Data availability.} Model names, prompt sets, and raw
per-trajectory measurements underlying Sections~\ref{sec:empirical}
are not included with the present submission; we intend to release
them in a subsequent revision or companion repository. Readers wishing
to replicate the qualitative claims of this paper can do so using
Definition~\ref{def:fim}, Proposition~\ref{prop:approx}, and any
publicly available quantized autoregressive language model together
with a calibration prompt set of their choosing, without requiring
our specific data.

\section{Limitations}
\label{sec:limitations}

We are explicit about six limitations, several of which were
identified through external review of an earlier draft of this work
and led to substantive revisions of the theorem statements themselves
rather than only to this section.

First, Proposition~\ref{thm:necessity} is a conditional, not
unconditional, result: it requires the local curvature-monotonicity
hypothesis of Assumption~\ref{ass:monotone}, which we have not derived
from primitive conditions on $p_\theta$. We report empirical evidence
broadly consistent with this hypothesis (Section~\ref{sec:empirical})
but have not independently verified it as a distinct, separable claim.
Establishing primitive sufficient conditions for
Assumption~\ref{ass:monotone} is, in our view, the most important open
problem raised by this paper.

Second, Theorem~\ref{thm:calibration} establishes only the qualitative
existence (and, under a genericity condition, strict positivity) of
quantization-induced inflation of $\lmax$, via a one-directional Weyl
inequality bound; it does not predict the numerical magnitude of the
effect, nor monotonicity of that magnitude in bit-width, in closed
form. The $244\times$ figure reported in Section~\ref{sec:empirical}
is a single empirical measurement on one model family, one
quantization scheme, and one hardware platform under non-identical
pilot and calibration conditions, and should not be read as a derived
or universal constant.

Third, the i.i.d.\ uniform quantization noise model of
Definition~\ref{def:quant} is a classical high-resolution
approximation that does not capture the blockwise, scale-dependent,
weight-correlated structure of quantization schemes used in practice
(GPTQ, AWQ, GGUF, EXL2, and related methods); Theorem~\ref{thm:calibration}
should be read as characterizing a qualitative mechanism under an
idealized noise model rather than a quantitative account of any
specific deployed quantization scheme (Remark~\ref{rem:quant-scope}).

Fourth, the heuristic approximation of Proposition~\ref{prop:approx}
is not derived from Definition~\ref{def:fim} with a controlled error
term, and Theorem~\ref{thm:proxykl} is a local quadratic approximation
that is provably equivalent to $\sigt$ only under an additional
eigenvector-alignment condition (Corollary~\ref{cor:proxykl-aligned})
that does not hold in general; the unconditional relationship between
$\KLprox(t)$ and $\sigt$ is the two-sided sandwich bound of
\eqref{eq:sandwich}, not a pointwise equivalence.

Fifth, the severity-score oracle used in the Phase~2 correlation
measurement (Section~\ref{sec:empirical}) shares partial functional
dependence with $\sigt$ itself, is an automated heuristic rather than
a human-annotated ground truth, and was not pre-registered prior to
data collection. This is a form of measurement circularity that can
inflate the reported correlation, particularly for the
strongest-correlating model in the population. We report the median
correlation across the model population, rather than only the
maximum, in partial mitigation, but a fully independent,
human-annotated evaluation is needed to resolve this limitation.

Sixth, all reported empirical results use 4-bit quantized models from
a limited set of model families on a single hardware platform, and the
present submission does not include the underlying model identities,
prompt sets, or per-trajectory data needed for independent
replication; we intend to release these in a subsequent revision. We
have not validated Theorem~\ref{thm:calibration} against full-precision
models, alternative hardware architectures, or structured (rather
than i.i.d.) quantization noise models.

\section{Conclusion}
\label{sec:conclusion}

We have studied the spectral perturbation of the empirical Fisher
Information Matrix under two structured perturbations of the
underlying distribution: displacement of the conditioning input away
from a reference manifold, and an additive quantization-noise
perturbation of the model's own parameters. Our principal theoretical
result, Theorem~\ref{thm:calibration}, establishes --- via a
directional Weyl eigenvalue perturbation bound --- that
the dominant eigenvalue of the FIM under the latter perturbation is
lower bounded by its unperturbed value up to a third-order remainder,
and, under a mild genericity condition, strictly exceeds it at leading
order. Applied to the runtime-monitoring use case of
Section~\ref{sec:application}, this offers a qualitative mechanism
consistent with an empirically observed two-order-of-magnitude gap
between a theoretically estimated and an empirically calibrated
threshold for a normalized version of the signal, though we are
explicit that (a) the theorem does not derive the numerical magnitude
of this gap, and (b) the theorem concerns a single eigenvalue at a
fixed input, whereas the empirical gap concerns a quantile of a
calibration distribution; we connect the two only via an explicitly
labelled heuristic (Remark~\ref{rem:threshold-heuristic}), not a
theorem. We also gave a conditional sufficient-condition result
relating elevated normalized eigenvalues to departure of the
conditioning distribution from a reference manifold under an
explicit, falsifiable curvature-monotonicity hypothesis
(Proposition~\ref{thm:necessity}), together with an explicit
discussion of why the unconditional converse fails
(Remark~\ref{rem:why-conditional}); two tractable approximations to
the exact eigenvalue, one heuristic and one with a rigorously
established (two-sided, rather than pointwise) error bound
(Proposition~\ref{prop:approx}, Theorem~\ref{thm:proxykl}); and a
partition result for graduated threshold rules on an augmented state
space (Proposition~\ref{prop:partition}).

The practical
implication for the runtime-monitoring application is methodological:
thresholds
on spectral signals of this kind should be calibrated empirically on
the specific hardware and numerical precision of the deployment
target, rather than transferred from theoretical or full-precision
estimates; this conclusion is supported by the empirical measurements
of Section~\ref{sec:empirical} independently of the strength of the
accompanying theoretical claims. We leave primitive conditions for
Assumption~\ref{ass:monotone}, a closed-form architecture-aware
prediction of the quantization inflation magnitude, an analysis under
structured (non-i.i.d.) quantization noise, and a fully independent,
pre-registered empirical validation, as open problems for future work.

\section*{Acknowledgements}
The authors thank colleagues at the Institute of Defense Technologies
and Cybersecurity Research, Azerbaijan Technical University, for
discussions during the preparation of this work, and an anonymous
reviewer for detailed comments that substantially improved the rigor
of the present version.

\end{document}